\documentclass{article}

\usepackage{microtype}
\usepackage{graphicx}
\usepackage{subfigure}
\usepackage{booktabs} %

\usepackage[accepted]{icml2021}

\usepackage{amsfonts}
\usepackage{amsmath}
\usepackage{amsthm}
\usepackage{amssymb}
\usepackage{dsfont}

\usepackage{xcolor}
\usepackage{color}
\usepackage{graphicx}

\usepackage[algo2e]{algorithm2e} 
\usepackage{verbatim}
\usepackage{xspace} 

\usepackage{enumerate}
\usepackage{enumitem}

\providecommand{\lin}[1]{\ensuremath{\left\langle #1 \right\rangle}}
\providecommand{\abs}[1]{\left\lvert#1\right\rvert}
\providecommand{\norm}[1]{\left\lVert#1\right\rVert}

  \providecommand{\R}{\mathbb{R}} 
  
  \DeclareMathOperator{\E}{{\mathbb E}}
  \providecommand{\EE}[2]{{\mathbb E}_{#1}\left.#2\right. }  
  \providecommand{\EEb}[2]{{\mathbb E}_{#1}\left[#2\right] } 

  \providecommand{\0}{\mathbf{0}}
  
  \renewcommand{\aa}{\mathbf{a}}
  \providecommand{\bb}{\mathbf{b}}

  \renewcommand{\gg}{\mathbf{g}}

  \providecommand{\nn}{\mathbf{n}}

  \providecommand{\uu}{\mathbf{u}}

  \providecommand{\xx}{\mathbf{x}}
  \providecommand{\yy}{\mathbf{y}}

  \providecommand{\mI}{\mathbf{I}}

  \providecommand{\cC}{\mathcal{C}}
  \providecommand{\cD}{\mathcal{D}}

  \providecommand{\cN}{\mathcal{N}}
  \providecommand{\cO}{\mathcal{O}}

  \usepackage{bm}

  \usepackage[textwidth=5cm]{todonotes}
  
\providecommand{\mycomment}[3]{\todo[caption={},size=footnotesize,color=#1!20]{\textbf{#2: }#3}}%
\providecommand{\inlinecomment}[3]{%
  {\color{#1}#2: #3}}%
\newcommand\commenter[2]%
{%
  \expandafter\newcommand\csname i#1\endcsname[1]{\inlinecomment{#2}{#1}{##1}}
  \expandafter\newcommand\csname #1\endcsname[1]{\mycomment{#2}{#1}{##1}}
}

\newtheorem{lemma}{Lemma}

\newtheorem{definition}{Definition}
\newtheorem{remark}[lemma]{Remark}
\newtheorem{assumption}{Assumption}
\newtheorem{theorem}[lemma]{Theorem}
\newtheorem{example}[lemma]{Example}

\usepackage{url}

\def\<#1,#2>{\left\langle #1,#2\right\rangle}
\newcommand{\RNum}[1]{\uppercase\expandafter{\romannumeral #1\relax}}

\commenter{seb}{red}
\commenter{ahmad}{purple}

\definecolor{mydarkblue}{rgb}{0,0.08,0.45}
\usepackage[breaklinks]{hyperref}
\hypersetup{colorlinks,allcolors=mydarkblue} %

\DeclareMathOperator*{\topk}{{top-k}}
\DeclareMathOperator{\randk}{{rand-k}}

\newcommand{\factor}{0.8}

\icmltitlerunning{On the Convergence of SGD with Biased Gradients}

\newcommand{\remove}[1]{}
\begin{document}

\twocolumn[
\icmltitle{On the Convergence of SGD with Biased Gradients} 

\icmlsetsymbol{equal}{*}

\begin{icmlauthorlist}
\icmlauthor{Ahmad Ajalloeian}{unil}
\icmlauthor{Sebastian U.~Stich}{epfl}
\end{icmlauthorlist}

\icmlaffiliation{unil}{University of Lausanne, Switzerland}
\icmlaffiliation{epfl}{EPFL, Switzerland}

\icmlcorrespondingauthor{SU Stich}{sebastian.stich@epfl.ch}

\icmlkeywords{Machine Learning, ICML}

\vskip 0.3in
]

\printAffiliationsAndNotice{}  %

\begin{abstract}
We analyze the complexity of biased stochastic gradient methods (SGD), where individual updates are corrupted by deterministic, i.e.\ \emph{biased} error terms. 
We derive convergence results for smooth (non-convex) functions and give improved rates under the Polyak-\L{}ojasiewicz condition.
We quantify how the magnitude of the bias impacts the attainable accuracy  and the convergence rates (sometimes leading to divergence).

Our framework covers many applications where either only biased gradient updates are available, or preferred, over unbiased ones for performance reasons.
For instance, in the domain of distributed learning, biased gradient compression techniques such as top-$k$ compression have been proposed as a tool to alleviate the communication bottleneck and in derivative-free optimization, only biased gradient estimators can be queried.
 We discuss a few guiding examples that show the broad applicability of our analysis. 
\end{abstract}

\section{Introduction}
\label{sec: intro}
The stochastic gradient descent (SGD) algorithm 
has proven to be effective for many machine learning applications. 
This first-order method has been intensively studied in theory and practice in recent years~\citep[cf.][]{Bottou2018:book}.
Whilst vanilla SGD crucially depends on unbiased gradient oracles, variations with \emph{biased} gradient updates have been considered in a few application domains recently. 

For instance, in the context of distributed parallel optimization where the data is split among several compute nodes, the standard mini-batch SGD updates can yield a bottleneck in large-scale systems and 
techniques such as structured sparsity~\cite{top-k_Alistarh_DBLP:conf/nips/AlistarhH0KKR18,Wangni2018:sparsification,sparsified_SGD_Stich:2018:SSM:3327345.3327357} or asynchronous updates~\cite{Recht2011:hogwild,Li2014:delay}
 have been proposed to reduce communication costs. However, sparsified or delayed SGD updates are no longer unbiased and the methods become more difficult to analyze~\cite{error_feedback_DBLP:journals/corr/abs-1909-05350,biased_compression_DBLP:journals/corr/abs-2002-12410}.

Another class of methods that do not have access to unbiased gradients are zeroth-order methods which find application for optimization of black-box functions~\cite{Nesterov:2017:RGM:3075317.3075511}, or for instance in deep learning for finding adversarial examples \cite{Universal_perturbation_2016arXiv161008401M,ZO_based_black_box_attackChen:2017:ZZO:3128572.3140448}. 
Some theoretical works that analyze zeroth-order training methods argue that the standard methods often operate with a \textit{biased} estimator of the true gradient~\cite{Nesterov:2017:RGM:3075317.3075511, ZO_SVRG_NIPS2018_7630}, in contrast to SGD that relies on unbiased oracles.

Approximate gradients naturally also appear in many other applications, such as in the
context of smoothing techniques, proximal updates or preconditioning~\cite{Aspremont2008:approximate,Schmidt2011:inexact,devolder_inexact_oracle_DBLP:journals/mp/DevolderGN14,Tappenden2016:inexact,Karimireddy2018:inexact}.

Many standard textbook analyses for SGD typically require unbiased gradient information to prove convergence~\cite{Lacoste2012:simpler,Bottou2018:book}. Yet, the manifold applications mentioned above, witness that SGD \emph{can} converge even when it has only access to inexact or biased gradient oracles. However, all these methods and works required different analyses which had to be developed separately for every application.
In this work, we reconcile insights that have been accumulated previously, explicitly or implicitly, and develop novel results, resulting in a clean framework for the analysis of a general class of biased SGD methods.

We study  SGD under a biased gradient oracle model. We go beyond the standard modeling assumptions that require unbiased gradient oracles with bounded noise~\cite{Bottou2018:book}, but allow gradient estimators to be biased, for covering the important applications mentioned above. We find that SGD with biased updates can converge to the same accuracy as  SGD with unbiased oracles when the bias is not larger than the norm of the expected gradient. Such a multiplicative bound on the bias does for instance hold for quantized gradient methods~\cite{Alistarh2017:qsgd}. %
Without such a relative bound on the bias, biased SGD in general only converges towards a neighborhood of the solution, where the size of the neighborhood depends on the magnitude of the bias.

Although it is a widespread view in parts of the literature that SGD does only converge with unbiased oracles, our findings shows that SGD \emph{does not need} unbiased gradient estimates to converge.

\paragraph{Algorithm and Setting.} 
Specifically, we consider unconstrained optimization problems of the form:
\begin{equation}
    f^{\star} := \min _{\xx \in \mathbb{R}^{d}} f(\xx)
\end{equation}
where $f \colon \R^d \to \R$ is a smooth function and 
where we assume that we only have access to \emph{biased} and \emph{noisy} gradient estimator. We analyze the convergence of (stochastic) gradient descent:  
\begin{equation}
\label{eq: def of g_t}
    \xx_{t+1} = \xx_t - \gamma_t \gg_t, \qquad  \gg_t = \nabla f(\xx_t) + \bb_t + \nn_t, 
\end{equation}
where $\gamma_t$ is a sequence of step sizes and $\gg_t$ is a (potentially biased) gradient oracle for zero-mean noise $\nn_t$, $\E \nn_t=\0$, and bias $\bb_t$ terms. In the case $\bb_t = \0$, the gradient oracle $\gg_t$ becomes unbiased and we recover the setting of SGD. If in addition $\nn_t = \0$ almost surely, we get back to the classic Gradient Descent (GD) algorithm.

\paragraph{Contributions.} We show that biased SGD methods can in general only converge to a neighborhood of the solution but indicate also interesting special cases where the optimum still can be reached. Our framework covers smooth optimization problems (general non-convex problems and non-convex problems under the Polyak-\L{}ojasiewicz condition) and unifies the rates for both deterministic and stochastic optimization scenarios. 
We discuss a host of examples that are covered by our analysis (for instance top-$k$ compression, random-smoothing) and compare the convergence rates to prior work.

\section{Related work}
\label{subsec: related work}
Optimization methods with deterministic errors have been previously studied mainly in the context of deterministic gradient methods. For instance,  \citet{Bertsekas2002:book} analyzes gradient descent under the same assumptions on the errors as we consider here, however does not provide concise rates. Similar, but slightly different, assumptions have been made in~\cite{Hu2020:biasedspd} that considers finite-sum objectives.

Most analyses of biased gradient methods have been carried out with specific applications in mind. For instance, 
computation of approximate (i.e.\ corrupted) gradient updates is for many applications computationally more efficient than computing exact gradients, and therefore there was a natural interest in such methods~\cite{Schmidt2011:inexact,Tappenden2016:inexact,Karimireddy2018:inexact}. 
\citet{Aspremont2008:approximate,Baes2009:estimate,devolder_inexact_oracle_DBLP:journals/mp/DevolderGN14} specifically also consider the impact of approximate updates on accelerated gradient methods and show these schemes have to suffer from error accumulation, whilst non-accelerated methods often still converge under mild assumptions.
\citet{devolder_inexact_oracle_DBLP:journals/mp/DevolderGN14,smooth_convex_opt_devolder_RePEc:cor:louvco:2011070} 
consider a notion of inexact gradient oracles that generalizes the notion of inexact subgradient oracles \cite{Polyak1987:book}.
 We will show later that their notion of inexact oracles is stronger than the oracles that we consider in this work.

In distributed learning, both unbiased~\cite{Alistarh2017:qsgd,Zhang2017:zipml} and biased~\cite{Dryden2016:topk,Aji2017:topk,top-k_Alistarh_DBLP:conf/nips/AlistarhH0KKR18} compression techniques have been introduced to address the scalability issues and communication bottlenecks. \citet{top-k_Alistarh_DBLP:conf/nips/AlistarhH0KKR18} analyze the top-$k$ compression operator which sparsifies the gradient updates by only applying the top $k$ components. They prove a sublinear rate that suffers a slowdown compared to vanilla SGD (this gap could be closed with additional assumptions). Biased updates that are corrected with error-feedback~\cite{sparsified_SGD_Stich:2018:SSM:3327345.3327357,error_feedback_DBLP:journals/corr/abs-1909-05350} do not suffer from such a slowdown. However, these methods are not the focus of this paper.
In recent work, \citet{biased_compression_DBLP:journals/corr/abs-2002-12410} analyze top-$k$ compression for deterministic gradient descent, often denoted as greedy coordinate descent~\cite{Nutini:2015vd,%
greedy_coordinate_descent_aistats/KarimireddyKSJ19}. We recover greedy coordinate descent convergence rates as a special case (though, we are not considering coordinate-wise Lipschitzness as in those specialized works).
Asynchronous gradient methods have sometimes been studied trough the lens
of viewing updates as biased gradients~\cite{Li2014:delay}, but the additional structure of delayed gradients admits more fine-grained analyses~\cite{Recht2011:hogwild,Mania2017:perturbed}.

Another interesting class of biased stochastic methods has been studied in~
\cite{Nesterov:2017:RGM:3075317.3075511} in the context of gradient-free optimization. They show that the finite-difference gradient estimator that evaluates the function values (but not the gradients) at two nearby points, provides a biased gradient estimator. As a key observation, they can show randomized smoothing estimates an \emph{unbiased} gradient of a different smooth function that is close to the original function. This observation is leveraged in the proofs. In this work, we consider general bias terms (without inherent structure to be exploited), and hence our convergence rates are slightly weaker for this special case.
Randomized smoothing was further studied in~\cite{Duchi:2012hr,bach16:smooth}.

\citet{bandit_biased_aistats/HuAGS16} study bandit convex optimization with biased noisy gradient oracles. They assume that the bias can be bounded by an absolute constant. Our notion is more general and covers uniformly bounded bias as a special case.
Other classes of biased oracle have been studied for instance by 
\citet{Hu2020:biased} who consider conditional stochastic optimization, \citet{Chen2018:biased} who consider consistent biased estimators and
\citet{biased_q_learning_aistats/WangG20a} who study the
statistical efficiency of Q-learning algorithms
by deriving finite-time convergence guarantees for a specific class of biased stochastic approximation scenarios.

\section{Formal setting and assumptions}
\label{sec: setting_assumptions}
In this section we discuss our main setting and assumptions. All our results depend on the standard smoothness assumption, and some in addition on the PL condition. 

\subsection{Regularity Assumptions}
\begin{assumption}[$L$-smoothness]
\label{general_smoothness_assumption}
    The function $f \colon \mathbb{R}^{d} \to \mathbb{R}$ is differentiable and there exists a constant $L>0$ such that for all $\xx,\yy \in \R^n$:
    \begin{equation}
    \label{eq: smoothness_equivalent_def}
        f(\yy) \leq f(\xx) + \lin{\nabla f(\xx), \yy-\xx} + \frac{L}{2} \norm{\yy-\xx}^{2}\,.
    \end{equation}
\end{assumption}

Sometimes we will assume the Polyak- \L{}ojasiewicz (PL)  condition (which is implied by standard $\mu$-strong convexity, but is much weaker in general):
\begin{assumption}[$\mu$-PL]
\label{general_strong_convex_assumption}
    The function $f \colon \mathbb{R}^{d} \to \mathbb{R}$ is differentiable and there exists a constant $\mu > 0$ such that
    \begin{equation}
    \label{eq:PL}
        \norm{\nabla f(\xx)}^2 \geq 2\mu (f(\xx) - f^\star)\,, \quad \forall \xx \in \R^d. 
    \end{equation}
\end{assumption}

\subsection{Biased Gradient Estimators}
\label{sec:biasedestimators}
Finally, whilst the above assumptions are standard, we now introduce biased gradient oracles.

\begin{definition}[Biased Gradient Oracle]
\label{def: biased gradient oracle}
A map $\gg\colon \R^d \times \cD \to \R^d$ s.t.
\begin{align*}
    \gg(\xx,\xi) = \nabla f(\xx) + \bb(\xx) + \nn(\xx,\xi)
\end{align*}
for a \emph{bias} $\bb \colon \R^d \to \R^d$ and zero-mean noise $\nn \colon \R^d \times \cD \to \R^d$, that is $\EE{\xi}\nn(\xx,\xi)=\0$, $\forall \xx \in \R^d$.
\end{definition}

We assume that the noise and bias terms are bounded:
\begin{assumption}[$(M, \sigma^2)$-bounded noise]
\label{assumption:noise}
There exists constants $M,\sigma^2 \geq 0$ such that
\begin{equation*}
    \EE{\xi\!}{\norm{\nn(\xx,\xi)}^2} \leq  M\norm{\nabla f(\xx) + \bb(\xx)}^2 + \sigma^2, \quad \forall \xx \in \R^d\,.
\end{equation*}
\end{assumption}
\begin{assumption}[$(m, \zeta^2)$-bounded bias]
\label{assumption: bias}
There exists constants $0 \leq m < 1$, and $\zeta^2 \geq 0$ such that
\begin{equation*}
    \norm{\bb(\xx)}^2 \leq  m\norm{\nabla f(\xx)}^2 + \zeta^2, \quad \forall \xx \in \R^d\,.
\end{equation*}
\end{assumption}
\begin{remark}
\label{rem: noise remark}
Assumptions~\ref{assumption:noise}--\ref{assumption: bias} and the inequality $\norm{\aa + \bb}^2 \leq 2 \norm{\aa}^2 + 2\norm{\bb}^2$ imply that it further holds
\begin{align*}
 \EE{\xi\!}{\norm{\nn(\xx,\xi)}^2} \leq  \bar M \norm{\nabla f(\xx)}^2 + \bar\sigma^2, \quad \forall \xx \in \R^d\,.
\end{align*}
for $\bar M\leq 2M(1+m)\leq 4M$ and $ \bar \sigma^2 \leq \sigma^2 + 2M \zeta^2$. 
\end{remark}
Our assumption on the noise is similar as in~\cite{bottou10.1007/978-3-7908-2604-3_16,unified_sgd_analisys}
and generalizes the standard bounded noise assumption (for which only $M=0$ is admitted). By allowing the noise to grow with the gradient norm and bias, an admissible $\sigma^2$ parameter can often be chosen to be much smaller than under the constraint $M=0$.

Our assumption on the bias is similar as in~\cite{Bertsekas2002:book}.
For the special case when $\zeta^2= 0$ one might expect (and we prove later) that gradient methods can still converge for any bias strength $0 \leq m < 1$, as in expectation the corrupted gradients are still aligned with the true gradient, $\E_\xi \lin{\nabla f(\xx_t),\gg(\xx,\xi)} > 0$ \citep[see e.g.][]{Bertsekas2002:book}.
In the case $\zeta^2> 0$, gradient based methods can only converge to a region where $\norm{\nabla f(\xx)}^2 = \cO(\zeta^2)$.
We make these claims precise in Section~\ref{sec: biased gradient framework} below.

A slightly different condition than Assumption~\ref{assumption: bias} was considered in~\cite{Hu2020:biasedspd}, but they measure the relative error with respect to the stochastic gradient $\nabla f(\xx) + \nn(\xx,\xi)$, and not with respect to $\nabla f(\xx)$ as we consider here. Whilst we show that biased gradient methods converge for any $m < 1$, for instance \citet{Hu2020:biasedspd} required a stronger condition $m \leq \frac{\mu}{L}$ on $\mu$-strongly convex functions (see also Remark~\ref{rem:stronglyconvex} below).

\section{Biased SGD Framework}
\label{sec: biased gradient framework}
In this section we study the convergence of stochastic gradient descent (SGD) with biased gradient oracles as introduced in Definition~\ref{def: biased gradient oracle}, see also Algorithm \ref{algo: SGD_algorithm}.
For simplicity, we will assume constant stepsize $\gamma_t=\gamma$, $\forall t$,  in this section.

\begin{algorithm}[htb]
   \caption{Stochastic Gradient Descent (SGD)}
   \label{algo: SGD_algorithm}
\begin{algorithmic}
   \STATE {\bfseries Input:} step size sequence  $(\gamma_t)_{t \geq 0}$, $\xx_0 \in \R^d$
   
   \FOR{$t=0$ {\bfseries to} $T-1$}
   \STATE compute biased stochastic gradient $\gg_t = \gg(\xx_t,\xi)$ 
   \STATE $\xx_{t+1} \gets \xx_{t} - \gamma_{t} \gg_t$
   \ENDFOR
   \STATE {\bfseries Return:} $\xx_{T}$
\end{algorithmic}
\end{algorithm}

\subsection{Modified Descent Lemma}
A key step in our proof is to derive a modified version of the descent lemma for smooth functions~\citep[cf.][]{nesterov_book_DBLP:books/sp/Nesterov04}. We give all missing proofs in the appendix.

\begin{lemma}\label{lemma:decent-smoothness}
Let $f$ be $L$-smooth, $\xx_{t+1}$, $\xx_t$ as in~\eqref{eq: def of g_t} with gradient oracle as in Assumptions~\ref{assumption:noise}--\ref{assumption: bias}. Then for any stepsize $\gamma \leq \frac{1}{(M+1)L}$ it holds
\begin{align}
\begin{split}
     &\E_{\xi} \left[ f(\xx_{t+1})  - f(\xx_t) \mid \xx_t \right] \\
     &\qquad \leq  \frac{\gamma (m-1)}{2} \norm{\nabla f(\xx_t)}^2 + \frac{\gamma}{2} \zeta^2 +  \frac{\gamma^2 L}{2} \sigma^2  
\end{split}
\label{eq:bysmoothness}
\end{align}
\end{lemma}

When $M=m=\zeta^2=0$ we recover (even with the same constants) the standard descent lemma.

\subsection{Gradient Norm Convergence}
We first show that Lemma~\ref{lemma:decent-smoothness} allows to derive convergence on all smooth (including non-convex) functions, but only with respect to gradient norm, as mentioned in Section~\ref{sec:biasedestimators} above.
By rearranging~\eqref{eq:bysmoothness} and the notation $F_{t}:=\E{f(\xx_t)-f^\star}$, $F=F_0$ we obtain
\begin{align*}
 \frac{(1-m)}{2} \E \norm{\nabla f(\xx_t)}^2 \leq \frac{F_t - F_{t+1}}{\gamma} + \frac{\zeta^2}{2} + \frac{\gamma L \sigma^2}{2}
\end{align*}
and by summing and averaging over $t$,
\begin{align*}
  \frac{1-m}{2}  \Psi_T \leq \frac{F_0}{T \gamma} + \frac{\zeta^2}{2} +  \frac{\gamma L \sigma^2}{2}
\end{align*}
where $\Psi_T := \frac{1}{T} \sum_{t=0}^{T-1} \E \norm{\nabla f(\xx_t)}^2$. 
We summarize this observation in the next lemma:
\begin{lemma}\label{lemma:new}
Under Assumptions~\ref{general_smoothness_assumption}, \ref{assumption:noise}, \ref{assumption: bias}, and for any stepsize $\gamma \leq \frac{1}{(M+1)L}$ it holds after $T$ steps of SGD:
\begin{align*}
 \Psi_{T} \leq \left(\frac{2F}{T\gamma (1-m)} + \frac{\gamma L \sigma^2}{1-m} \right) + \frac{\zeta^2}{1-m}\,,
\end{align*}
where $\Psi_T$ is defined as above.
\end{lemma}
The quantity on the left hand side is equal to the expected gradient norm of a uniformly at random selected iterate, $\Psi_T = \E \norm{\nabla f(\xx_{\rm out})}^2$ for $\xx_{\rm out} \in_{\rm u.a.r.} \{\xx_0,\dots,\xx_{T-1}\}$ and a standard convergence measure. Alternatively one could consider $\min_{t \in [T]} \E \norm{\nabla f(\xx_t)}^2 \leq \Psi_T$.

If there is no additive bias, $\zeta^2=0$, then Lemma~\ref{lemma:new} can be used to show convergence of SGD by carefully choosing the stepsize to minimize the term in the round bracket. However, with a bias $\zeta^2 > 0$ SGD can only converge to a $\cO(\zeta^2)$ neighborhood of the solution. 

\begin{theorem}\label{thm:smooth}
Under Assumptions~\ref{general_smoothness_assumption}, \ref{assumption:noise}, \ref{assumption: bias}, and by choosing the stepsize
$\gamma=\min \left\{ \frac{1}{(M+1)L}, \frac{\epsilon (1-m) + \zeta^2}{2 L \sigma^2} \right\}$ for $\epsilon \geq 0$, then
\begin{align*}
    T =  \cO \left(\frac{M+1}{\epsilon(1-m) + \zeta^2} 
     + \frac{\sigma^2}{\epsilon^2 (1-m)^2 + \zeta^4}   \right) \cdot  LF
\end{align*}
iterations are sufficient to obtain $\Psi_T = \cO \bigl( \epsilon + \frac{\zeta^2}{1-m} \bigr)$.
\end{theorem}

If there is no bias, $\zeta^2=m=0$, then this result recovers the standard convergence rate of gradient descent (when also $\sigma^2=M=0$) or stochastic gradient descent, in general.
When the biased gradients are correlated with the gradient (i.e. $\zeta^2=0$, $m>0$), then the exact solution can still be found, though the number of iterations increases by a factor of $(1-m)$ in the deterministic case, or by $(1-m)^2$ if the noise term is dominating the rate.
When there is an uncorrelated bias, $\zeta^2>0$, then SGD can only converge to a neighborhood of a stationary point and a \emph{finite number of iterations} suffice to converge. This result is tight in general, as we show in Remark~\ref{rem:error-floor} below. 

\begin{remark}[Tightness of Error-floor]
\label{rem:error-floor}
Consider a $L$-smooth function $f \colon \R^d \to \R$ with  gradient oracle 
$\gg(\xx)= \nabla f(\xx) + \rho(\xx)\bb$, for $0 \leq m \leq 1$, $\bb \in \R^d$ with $\norm{\bb}^2=\zeta^2$ and $\rho(\xx)^2:= 1 + \frac{m}{\zeta^2}\norm{\nabla f(\xx)}^2$.
 This oracle has $(m,\zeta^2)$-bounded bias. 
 For any stationary point $\xx^\star$ with $\nabla \gg(\xx)=0$, it holds $\nabla f(\xx^\star)=-\rho(\xx^\star)\bb$, i.e. $\norm{\nabla f(\xx^\star)}^2 = \frac{\zeta^2}{1-m}$.
\end{remark}

\subsection{Convergence under PL condition}
\label{subsec: convergence PL}
To show convergence under the PL condition, we observe that~\eqref{eq:bysmoothness} and~\eqref{eq:PL} imply
\begin{align}
    F_{t+1} \leq \left(1-\gamma \mu(1-m) \right) F_t + \frac{\gamma \zeta^2}{2} + \frac{\gamma^2 L \sigma^2 }{2}\,. \label{eq:recusion}
\end{align}
By now unrolling this recursion we obtain the next theorem.

\begin{theorem}
\label{thm: convergence under PL condition}
Under Assumptions~\ref{general_smoothness_assumption}--%
\ref{assumption: bias} it holds for any stepsize $\gamma \leq \frac{1}{(M+1)L}$ that
\begin{align*}
 F_{T} \leq (1-\gamma \mu (1-m))^T F_0 + \frac{1}{2} \Xi\,,
\end{align*}
where 
\begin{align}
\Xi := \frac{ \zeta^2}{\mu (1-m)} + \frac{\gamma L \sigma^2}{\mu (1-m)}\,. \label{eq:floor}
\end{align}
Consequently, by choosing the stepsize $\gamma = \min\bigl\{\frac{1}{(M+1)L}, \frac{\epsilon \mu (1-m) + \zeta^2}{L \sigma^2} \bigr\}$, for $\epsilon \geq 0$, %
\begin{align*}
 T = \tilde \cO \left( (M+1) \log \frac{1}{\epsilon} + \frac{\sigma^2}{\epsilon \mu (1-m) + \zeta^2} \right) \cdot \frac{\kappa}{1-m}
\end{align*}
iterations suffice for $F_T= \cO \bigl(\epsilon + \frac{\zeta^2}{\mu (1-m)} \bigr)$. Here $\kappa := \frac{L}{\mu}$ and the $\tilde \cO(\cdot)$ notation hides logarithmic factors.
\end{theorem}
This result shows, that SGD with constant stepsize converges up to the error floor $\cO(\zeta^2 + \gamma L \sigma^2)$, given in~\eqref{eq:floor}. The second term in this bound can be decreased by choosing the stepsize $\gamma$ small, however, the first term remains constant, regardless of the choice of the stepsize. We will investigate this error floor later in the experiments.

Without bias terms, we recover the best known rates under the PL condition~\cite{karimi16:pl}. Compared to the SGD convergence rate for $\mu$-strongly convex function our bound is worse by a factor of $\kappa$  compared to~\citep{unified_sgd_analisys}, but we consider convergence of the function value of the last iterate $F_T$ here. 

\begin{remark}[Convergence on Strongly Convex Functions.]
\label{rem:stronglyconvex}
Theorem~\ref{thm: convergence under PL condition} applies also to $\mu$-strongly convex functions (which are $\mu$-P\L{}) and shows convergence in function value for any $0 \leq m < 1$ (if $\zeta^2=0$). By $\mu$-strong convexity this implies iterate convergence $\norm{\xx_t-\xx^\star}^2 \leq \frac{2}{\mu} F_t$, however, at the expense of an additional $\frac{1}{\mu}$ factor. For proving a stronger result, for instance to prove a one-step progress bound for $\norm{\xx_t-\xx^\star}^2$ (as it standard in the unbiased case~\cite{Lacoste2012:simpler,unified_sgd_analisys}) it is necessary to impose a strong condition $m \leq \frac{1}{\kappa}$, similar as in~\cite{Hu2020:biasedspd}. 
\end{remark}

\subsection{Divergence on Convex Functions}
Albeit Theorem~\ref{thm:smooth} shows convergence of the gradient norms, this convergence does not imply convergence in function value in general (unless the gradient norm is related to the function values, as for instance guaranteed by~\eqref{eq:PL}).
We now give an example of a weakly-convex (non-strongly convex) smooth function, where Algorithm~\ref{algo: SGD_algorithm} can diverge.
\begin{example}\label{ex:huber}
Consider the Huber-loss function $h \colon \R \to \R$,
\begin{align*}
 h(x) = \begin{cases} \abs{x}, & \abs{x}>1 \\ \frac{1}{2}x^2 + \frac{1}{2} & \abs{x} \leq 1
    \end{cases}
\end{align*}
and define the gradient oracle $g(x)=h'(x)-2$. This is a biased oracle with $\zeta^2 \leq 4$, and it is easy to observe that iterations~\eqref{eq: def of g_t} diverge for any stepsize $\gamma > 0$, given $x_0 >1$.
\end{example}

This result shows that in general it is not possible to converge when $\zeta^2\neq 0$, and that the approximate solutions found by SGD with unbiased gradients (in Example~\ref{ex:huber} converging to 0) and by SGD with biased gradients can be arbitrarily far away.
It is an interesting open question, whether SGD can be made \emph{robust to biased updates}, i.e. preventing this diverging behaviour through modified updates (or stronger assumptions on the bias). We leave this for future work.

\section{Discussion of Examples}
\label{sec: special cases}
In this section we are going to look into some examples which can be considered as special cases of our biased gradient framework. Table \ref{tab: special cases} shows a summary of these examples together with value of the respective parameters. %
\begin{table*}[!t]
\caption{Special cases of our biased gradient framework. Each column represents the value of our framework parameters in an example. Here $\nn(\xx)$ is a $(M,\sigma^2)$-bounded noise term and $\uu$ isotropic Gaussian with variance 1.}
\label{tab: special cases}
\centering
\resizebox{0.95\linewidth}{!}{
\begin{tabular}{llccccc} 
\toprule
                      & def                                                                                    & M & $\sigma^2$ & 1-m & $\zeta^2$              \\ \midrule
top-$k$ compression & $\gg(\xx) = \topk(\nabla f(\xx))$  & $0$  & $0$ & $\frac{k}{d}$ & 0 \\   
random-$k$ compression & $\gg(\xx) = \randk(\nabla f(\xx))$   & $0 $ & 0 & $\frac{k}{d}$ & 0  \\
random-$k$ compression (of stochastic g.) & $\gg(\xx) = \randk(\nabla f(\xx) + \nn(\xx))$   & $ (1+M)\frac{d}{k} - 1$ & $\frac{k}{d} \sigma^{2}$ & $\frac{k}{d}$ & 0  \\
Gaussian smoothing & $\gg_{\rm GS}(\xx) = \frac{f(\xx+ \tau \uu) - f(\xx)}{\tau} \uu$                                                                         &  4(d+4)  &  $3\tau^2 L^2 (d+4)^3$          & 1  &    $\frac{\tau^2}{4}L^2(d+3)^{3}$                    \\
$(\delta, L)$-oracle  &  (see text)    & 0   &     0       & 1  &   $2\delta L$                     \\
 stochastic $(\delta, L)$-oracle  &  (see appendix)    & 0   &     $\sigma^2$       & 1  &   $2\delta L$                     \\
compressed gradient   & $\gg(\xx) = \mathcal{C}(\nabla f(\xx)), \quad \mathcal{C} \in C(\delta)$ &   0 & 0          & $ \delta$ & 0 \\ 
\bottomrule
\end{tabular}
}
\end{table*}

\subsection{Top-$k$ and Random-$k$ sparsification}
The well-known top-$k$ spasification operator~\cite{Dryden2016:topk,top-k_Alistarh_DBLP:conf/nips/AlistarhH0KKR18} that selects the $k$ largest coordinates of the gradient vector $\nabla f(\xx)$
is a biased gradient oracle with $m \leq \frac{d-k}{d}$ and $\zeta^2=0$. This means, that GD with top-$k$ compression converges to a stationary point on smooth functions and under the PL condition the convergence rate is $\cO \bigl( \frac{d}{k} \kappa \log \frac{1}{\epsilon} \bigr)$ when there is no additional noise. This recovers the rate of greedy coordinate descent when analyzed in the Euclidean norm~\cite{Nutini:2015vd}.

The biased random-$k$ sparsification operator  randomly selects $k$ out of the $d$ coordinates of the stochastic gradient, and sets the other entries to zero.
As $m=1-\frac{k}{d}$, this implies a $\frac{d}{k}$ slowdown over the corresponding rates of GD (recovering the rate of random coordinate descent~\cite{Nesterov:2017:RGM:3075317.3075511,biased_rand_kDBLP:journals/mp/RichtarikT16}) and
SGD, with asymptotic dominant term $\cO\bigl(\frac{d \sigma^2}{k}\bigr)$
that is expected~\citep[cf.][]{Chaturapruek2015:noise}. In the appendix we further discuss differences to \emph{unbiased} random-$k$ sparsification.\footnote{By rescaling the obtained sparse vector by the factor $\frac{d}{k}$ one obtains a \emph{unbiased} estimator (with higher variance). With these sparsification operators the dominant term in the rate remains  $\cO\bigl(\frac{d \sigma^2}{k}\bigr)$, but the optimization terms can be improved as a benefit of choosing larger stepsizes. However, this observation does not show a limitation of the biased rand-$k$ oracle, but is merely a consequence of applying our \emph{general} result  to this special case. Both algorithms are equivalent when properly tuning the stepsize.
}

\subsection{Zeroth-Order Gradient}
\label{subsec: ZO}
Next, we discuss the zeroth-order gradient oracle obtained by Gaussian smoothing~\citep[cf.][]{Nesterov:2017:RGM:3075317.3075511}. This oracle is defined as \cite{Polyak1987:book}:
\begin{equation}
    \label{eq: grad free oracle}
    \gg_{\rm GS}(\xx) = \frac{f(\xx+ \tau \uu) - f(\xx)}{\tau} \cdot \uu
\end{equation}
where $\tau > 0$ is a smoothing parameter and $\uu \sim \cN(\0,\mI)$ a random Gaussian vector.
\citet[Lemma 3 \& Theorem 4]{Nesterov:2017:RGM:3075317.3075511} provide estimates for the 
bias: %
 \begin{equation*}
     \norm{\mathbb{E}\gg_{\rm GS}(\xx) - \nabla f(\xx)}^2 \leq \frac{\tau^2}{4}L^2(d+3)^3\,,
 \end{equation*}
and the second moment:
    \begin{equation*}
    \begin{split}
    \label{eq: upper bound on 2nd moment of ZO}
    \mathbb{E} \norm{\gg_{GS} (\xx)}^2 & \leq 3\tau^{2}L^{2}(d+4)^{3} +  4(d+4) \norm{\E {\gg_{GS}(\xx)}}^{2}  \,.
    \end{split}
    \end{equation*}
We conclude %
that Assumptions \ref{assumption:noise}, and \ref{assumption: bias} hold with:
    \begin{align*}
        &\zeta^2 = \frac{\tau^2}{4}L^2(d+3)^3, \quad   m=0 \\ 
        &\sigma^2= 3\tau^2 L^2 (d+4)^3, \quad M= 4(d+4) \,.
    \end{align*}
\citet{Nesterov:2017:RGM:3075317.3075511} show that on smooth functions gradient-free oracles in general slow-down the rates by a factor of $\cO(d)$, similar as we observe here with $M=\cO(d)$. However, their oracle is stronger, for instance they can show convergence (up to $\cO(\zeta^2)$) on weakly-convex functions, which is not possible with our more general oracle (see Example~\ref{ex:huber}).

\subsection{Inexact gradient oracles}
\citet{devolder_inexact_oracle_DBLP:journals/mp/DevolderGN14} introduced the notion
of inexact gradient oracles. A $(\delta, L)$-gradient oracle for $\yy \in \R^d$ is a pair $(\tilde{f}(\yy),\tilde{\gg}(\yy))$ that satisfies $\forall \xx \in \R^d$:
\begin{align*}
    0 \leq f(\xx)-\left(\tilde{f}(\yy)+\lin{ \tilde{\gg}(\yy), \xx-\yy}\right) \leq \frac{L}{2}\norm{\xx-\yy}^{2}+\delta\,.
\end{align*}
We have by \cite{devolder_inexact_oracle_DBLP:journals/mp/DevolderGN14}:
\begin{align*}
    & \norm{\bb(\xx)}^2 = \norm{\nabla f(\xx) - \tilde{\gg}(\xx)}^2 \leq 2 \delta L \, %
\end{align*}
hence we can conclude: %
$\zeta^2 = 2\delta L$ and $m=0$.

It is important to observe that the notion of a $(\delta,L)$ oracle is stronger than what we consider in Assumption~\ref{assumption: bias}. For this, consider again the Huber-loss from Example~\ref{ex:huber} and a $(\delta,L)$ gradient estimator $\tilde{\gg}(x)$. For $x \to \infty$, we observe that it must hold $\norm{\tilde{\gg}(x)- \nabla h(x)} \to 0$, otherwise the condition of the $(\delta,L)$ oracle is violated. In contrast, the bias term in Assumption~\ref{assumption: bias} can be constant, regardless of $x$.
\citet[][preprint]{smooth_convex_opt_devolder_RePEc:cor:louvco:2011070} generalized the notion of $(\delta,L)$ oracles to the stochastic case, which we discuss in the appendix.

\subsection{Biased compression operators}
The notion of top-$k$ compressors has been generalized to arbitrary $\delta$-compressors, defined as
\begin{definition}[$\delta$-compressor]
$\mathcal{C} \in C\left(\delta\right)$ if $\exists \delta > 0$ s.t.:
\begin{align*}
    \mathbb{E} \left[\norm{\mathcal{C} \left( \gg \right) - \gg}^2\right] \leq \left(1 - \delta \right) \norm{\gg}^2, \quad \forall \gg \in \R^d\,.
\end{align*}
\end{definition}
This notion has for instance been used in \cite{biased_compression_DBLP:journals/corr/abs-2002-12410, sparsified_SGD_Stich:2018:SSM:3327345.3327357}.
For this class of operator it holds
\begin{align*}
    m \leq 1 - \delta , \qquad \zeta^2 = 0\,.
\end{align*}
In the noiseless case, our results show convergence for arbitrary small $\delta > 0$ ($m < 1)$ and we recover the rates given in \citep[Theorem 1]{biased_compression_DBLP:journals/corr/abs-2002-12410}.

\subsection{Delayed Gradients}
Another example of updates that can be viewed as biased gradients, are delayed gradients, for instance arising in asynchronous implementations of SGD~\cite{Recht2011:hogwild}. For instance \citet[Proof of Theorem 2]{Li2014:delay} provide a bound on the bias. See also~\cite{error_feedback_DBLP:journals/corr/abs-1909-05350} for a more refined analysis in the stochastic case. However, these bounds are not of the form as in Assumption~\ref{assumption: bias}, as they depend on past gradients, and not only the current gradient at $\xx_t$. We leave generalizations of Assumption~\ref{assumption: bias} for future work (moreover, delayed gradient methods have been analyzed tightly with other tools already).

\section{Experiments}
In this section we verify numerically whether our theoretical bounds are aligned with the numerical performance of SGD with biased gradient estimators.
In particular, we study whether the predicted error floor~\eqref{eq:floor} of the form $\cO(\zeta^2 + \gamma L \sigma^2)$ can be observed in practice.

\paragraph{Synthetic Problem Instance.}
We consider a simple least squares objective function $f(\xx) = \frac{1}{2} \norm{A\xx}_{2}^{2}$, in dimension $d=10$, where the matrix $A$ is chosen according to the Nesterov worst function \cite{nesterov_book_DBLP:books/sp/Nesterov04}. We choose this function since its Hessian, i.e., $A^T A$ has a very large condition number and hence it is a hard function to optimize. By the choice of $A$ this objective function is strongly convex. We control the stochastic noise by adding Gaussian noise $\nn_t \sim \cN(0, \sigma^2)$ to every gradient.

\subsection{Experiment with a synthetic additive bias}
\label{subsec: synthetic bias}
 In our first setup (Figure \ref{fig: constant_stepsize_biased}), we consider a synthetic case in which we add a constant value $\zeta \in \{0, 0.1, 0.001\}$ to the gradient to control the bias, and vary the noise $\sigma^2 \in \{0,1\}$. This corresponds to the setting where we have $m=0, \zeta^2 > 0$, $M=0$, $\sigma^2 \geq 0$. 
\textbf{Discussion of Results.} Figure \ref{fig: constant_stepsize_biased} highlights the effect of the parameters $\sigma^2$ and $\zeta^2$ on the rate of the convergence and the neighbourhood of the convergence. 
Notice that a small bias value can only affect the convergence neighborhood when there is no stochastic noise. In fact this matches the error floor we derived in \eqref{eq:floor}. When $\zeta^2$ is small compared to $\sigma^2$, i.e, the bias is insignificant compared to the noise, the second term in \eqref{eq:floor} determines the neighborhood of the optimal solution to which SGD with constant stepsize converges. In contrast, when $\zeta$ is large the first term in \eqref{eq:floor} dominates the second term and determines the neighborhood of the convergence. As mentioned earlier in Section \ref{subsec: convergence PL}, the second term in \eqref{eq:floor} also depends on the stepsize, meaning that by increasing (decreasing) the stepsize the effect of $\sigma^2$ in determining the convergence neighborhood can be increased (decreased).

\begin{figure}
    \centering
    \includegraphics[width=\factor\linewidth]{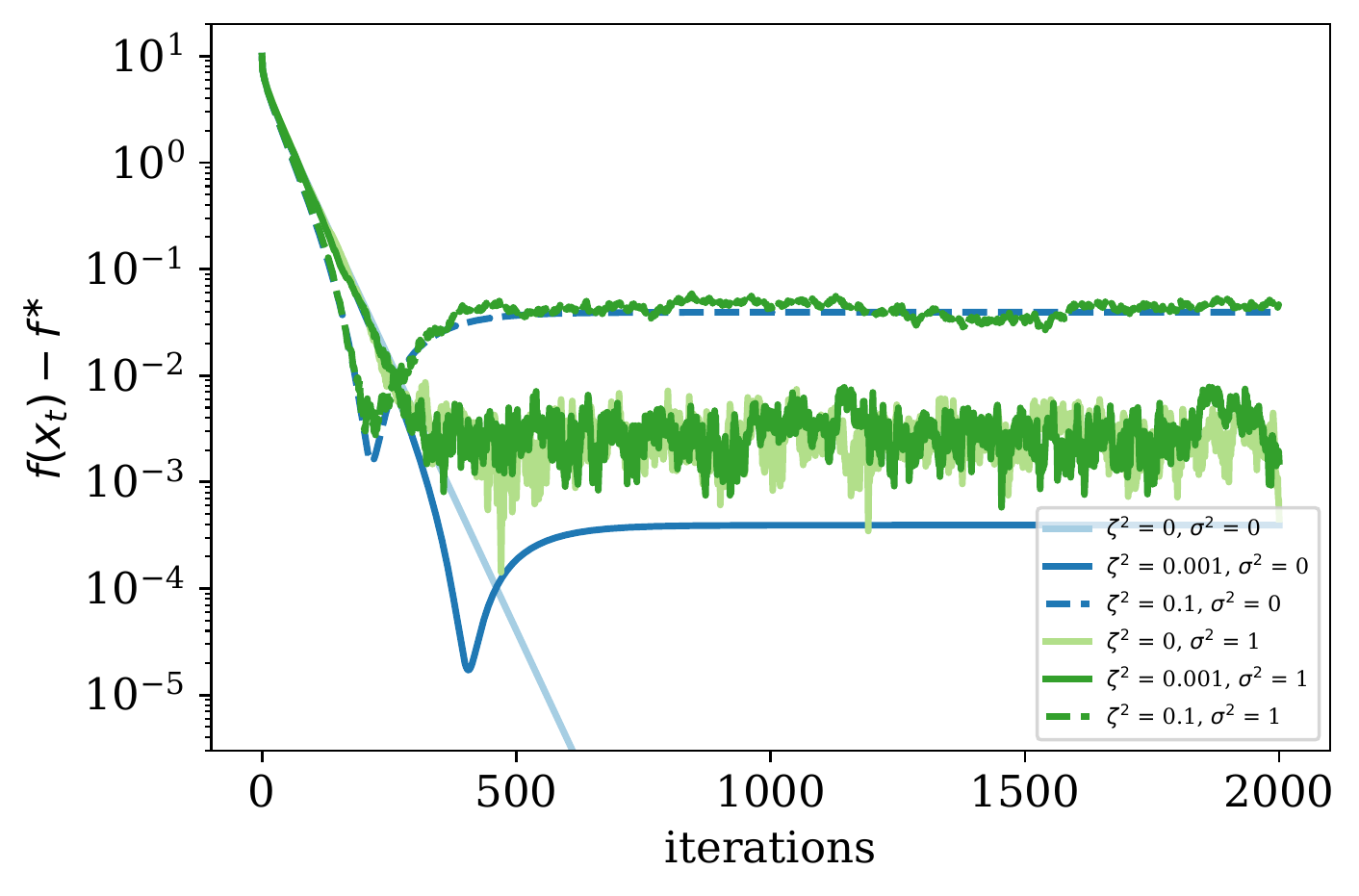}
    \caption{Effect of the parameters $\zeta$, and $\sigma$ on the error floor when optimizing $f(\xx), \xx \in \R^{10}$. Here we use a fixed stepsize $\gamma = 0.01$ and we set $M=m=0$.}
    \label{fig: constant_stepsize_biased}
\end{figure}

\begin{figure}
    \centering
    \includegraphics[width=\factor\linewidth]{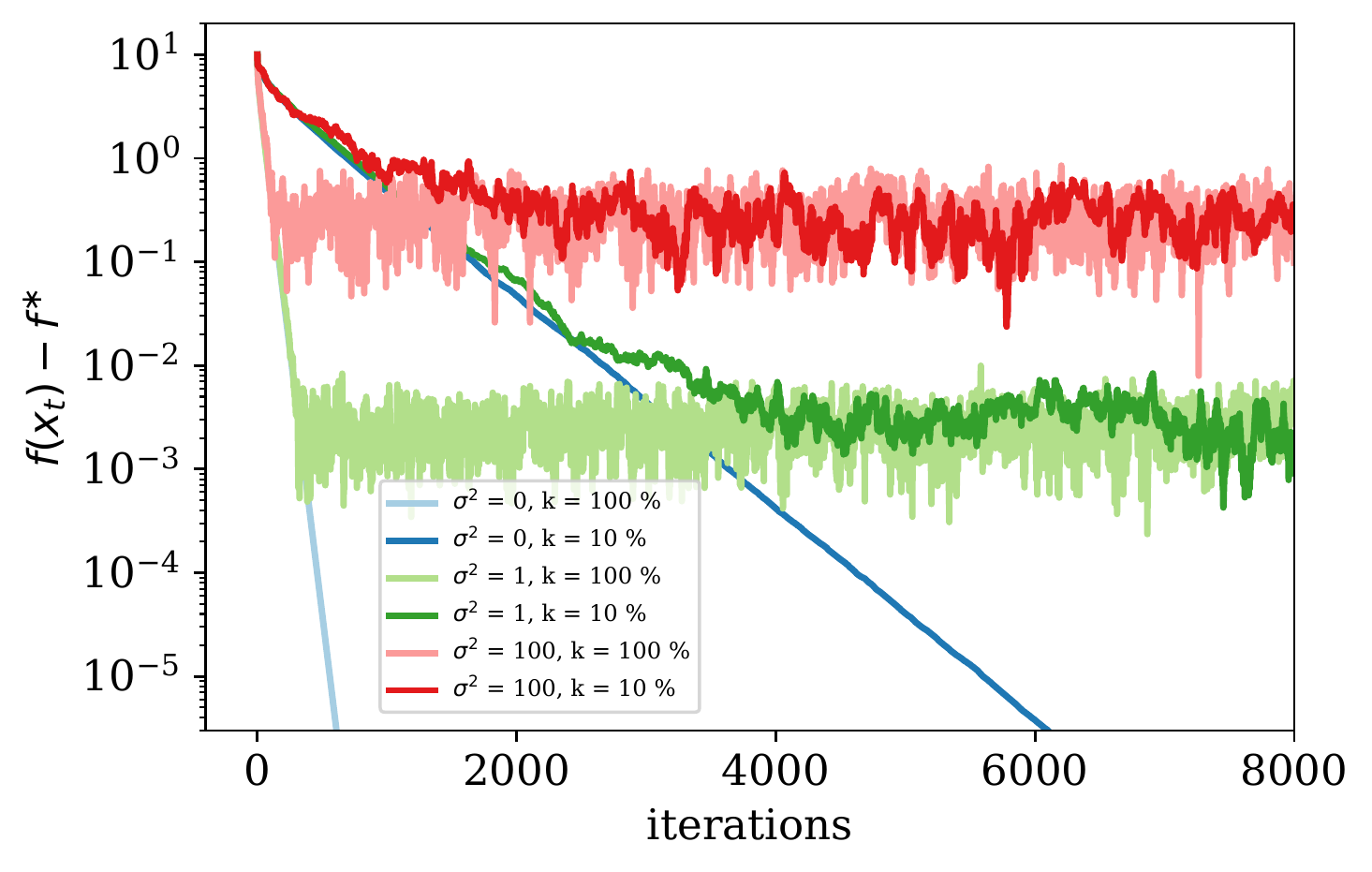}
    \caption{Effect of the parameters $\sigma^2$ and $k$ which change the noise and bias level respectively, for random-$k$ sparsification. Here we optimize $f(\xx), \xx \in \R^{10}$ and use a fixed stepsize $\gamma = 0.01$.}
    \label{fig: constant_stepsize_randk}
\end{figure}

\begin{figure}
    \centering
    \includegraphics[width=\factor\linewidth]{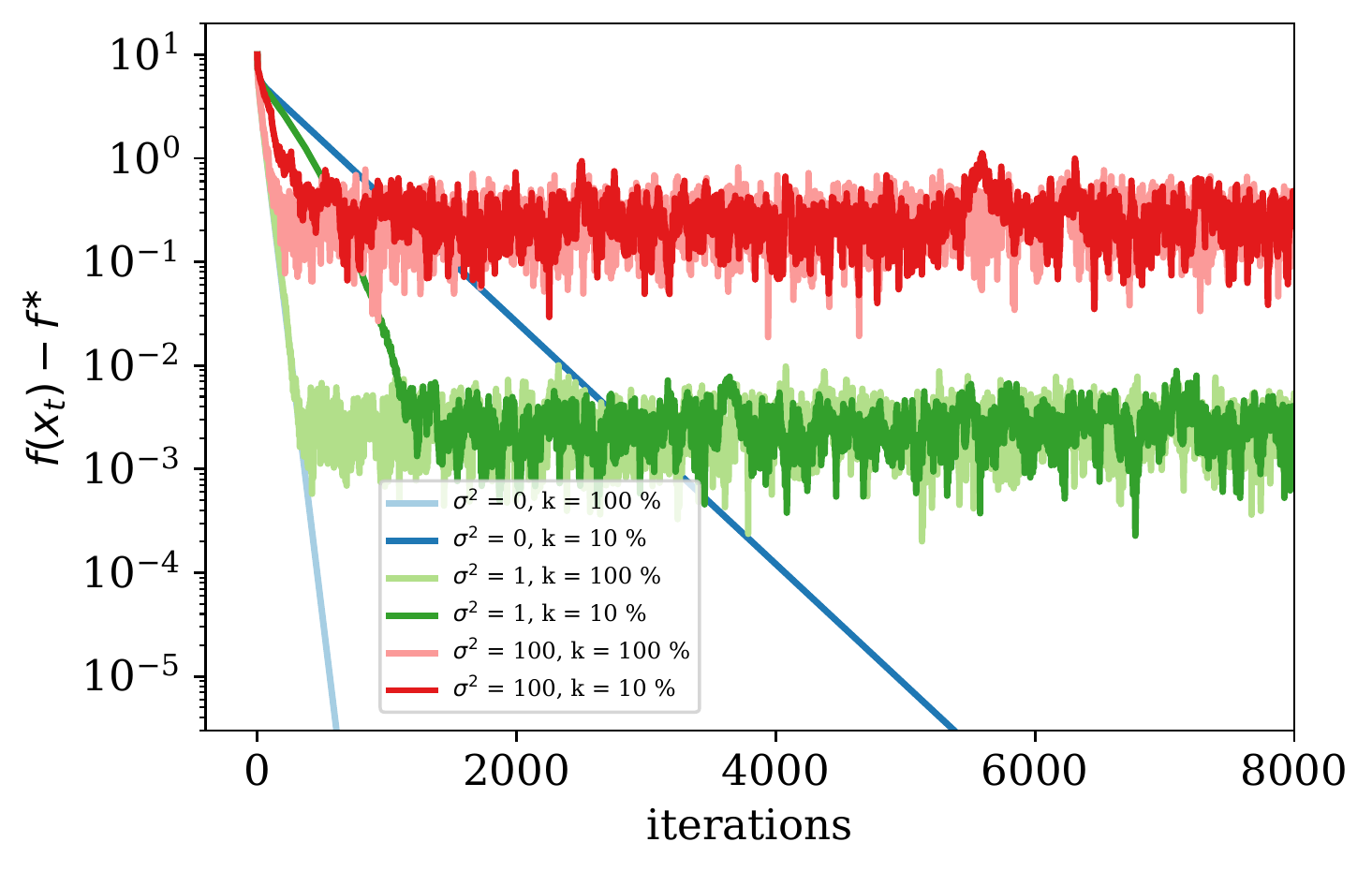}
    \caption{Effect of the parameters $\sigma^2$ and $k$ which change the noise and bias level respectively, for top-$k$ sparsification. Here we optimize $f(\xx), \xx \in \R^{10}$ and use a fixed stepsize $\gamma = 0.01$.}
\label{fig: constant_stepsize_topk}
\end{figure}

\subsection{Experiment with Top-$k$ and Random-$k$ sparsification}
In contrast to the previous section, we leave the completely synthetic setup, and consider structured bias that arises from compressing the gradients by using either the random-$k$ or top-$k$ method~\cite{top-k_Alistarh_DBLP:conf/nips/AlistarhH0KKR18}, displayed in Figures~\ref{fig: constant_stepsize_randk}, \ref{fig: constant_stepsize_topk}.
We use the same objective function $f$ as introduced earlier. In each iteration we compress the stochastic gradient, $\sigma^2 \in \{0,1,100\}$, by the top-$k$ or random-$k$ method, i.e.,
\begin{align*}
    \gg(\xx) = \cC (\nabla f(\xx) + n(\xx)), \quad n(\xx) \sim \cN (0, \sigma^2)
\end{align*}
where $\cC$ is either $\topk$ or $\randk$ operator, for $\frac{k}{d} \in \{0.1,1\}$.
As we showed in Section~\ref{sec: special cases}, this corresponds to the setting $\zeta^2=0$ for rand-$k$, and $\zeta^2>0$ for top-$k$, with $m=1-\frac{k}{d}$.

\textbf{Error floor.} Figure \ref{fig: constant_stepsize_randk}, \ref{fig: constant_stepsize_topk} show the effect of the parameters $\sigma^2$ and $k$ (that determine the level of noise and bias respectively) on the rate of convergence as well as the error floor, when using random-$k$ and top-$k$ compressors and a constant stepsize $\gamma$.

\textbf{Convergence rate.} 
In Figure \ref{fig: six graphs} we show the convergence for different noise levels $\sigma^2$ and compression parameters $k$. For each configuration and algorithm, we tuned the stepsize $\gamma$ to reach the target accuracy (set by $\epsilon$) in the fewest number of iterations as possible.

\paragraph{Discussion of Results.} 
 We can see that in this case even in the presence of stochastic noise, random-$k$ and top-$k$, although with a slower rate, can still converge to the same level as with no compression.
We note that top-$k$ consistently converges faster than rand-$k$, for the same values of $k$.
The error floor for both schemes is mainly determined by the strength of the noise $\sigma^2$ that dominates in~\eqref{eq:floor}.
In Figure \ref{fig: six graphs} we see that top-$k$ and random-$k$ can still converge to the same error level as no compression scheme but with a slower rate.
The convergence rate of the compression methods become slower when decreasing the compression parameter $k$, matching the rate we derived in Theorem \ref{thm: convergence under PL condition}.

\begin{figure}
    \centering
    \includegraphics[width=\factor\linewidth]{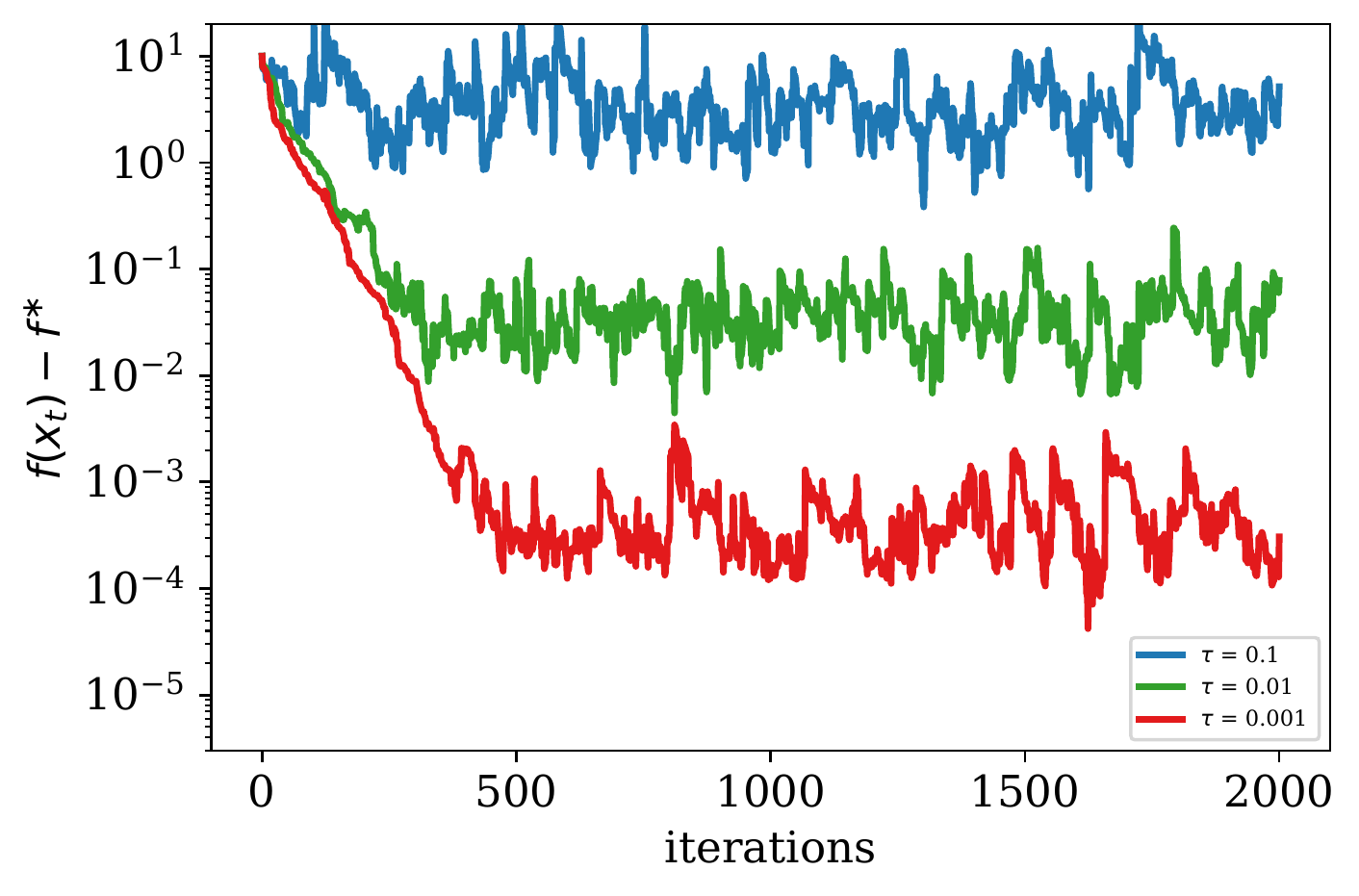}
    \caption{Effect of the parameter $\tau$ (the smoothing parameter) on the error floor, when zeroth-order gradient is used. Here we optimize $f(\xx), \xx \in \R^{10}$ and use a fixed stepsize $\gamma = 0.01$.}
\label{fig: constant_stepsize_ZO}
\end{figure}

\begin{figure*}
    \centering
    \includegraphics[width=\factor\textwidth]{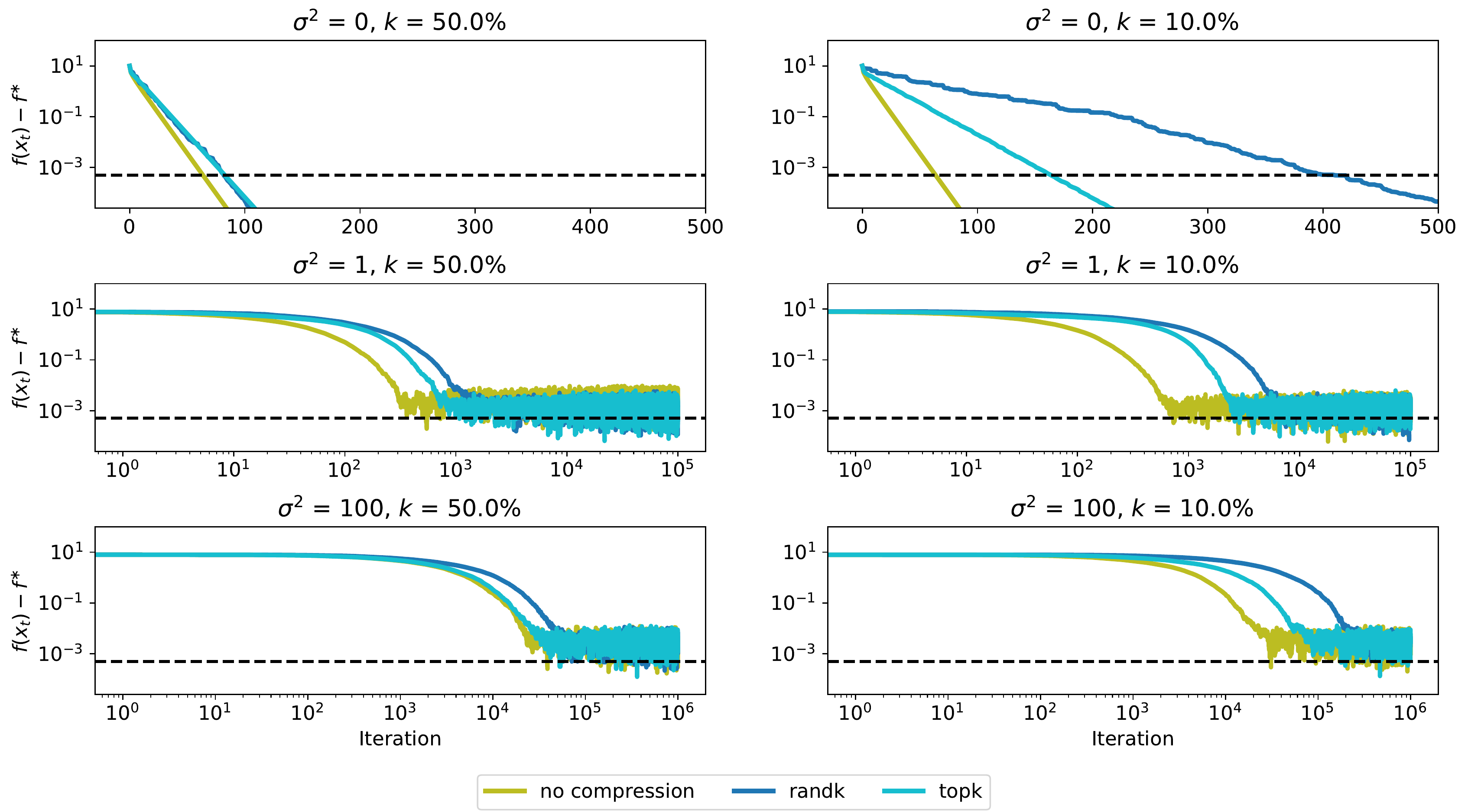}
    \caption{Convergence of $f(\xx)=\frac{1}{2}\norm{Ax}^2, \quad \xx \in \mathbb{R}^{10}$ to target accuracy $\epsilon = 5 \times 10^{-4}$ for different problem settings ($\sigma^2$ increasing to the bottom and $k$ decreasing to the left), and different compression methods. Stepsizes were tuned for each setting individually to reach target accuracy in as few iterations as possible.}
    \label{fig: six graphs}
\end{figure*}

\subsection{Experiment with zeroth-order gradient oracle}
We optimize the function $f$ using the zeroth-order gradient oracle defined in Section \ref{subsec: ZO}. The smoothing parameter $\tau$  affects both the values of $\sigma^2$ and $\zeta^2$ and hence the error floor. \looseness=-1

\textbf{Discussion of Results.} 
Figure \ref{fig: constant_stepsize_ZO} shows the error floor for the zeroth-order gradient method with different values of $\tau$. We can observe that the difference between the error floors matches what we saw in theory, i.e, $\sigma^2 + \zeta^2 \sim \cO(\tau^2)$.

\begin{figure}
    \centering
    \includegraphics[width=\factor\linewidth]{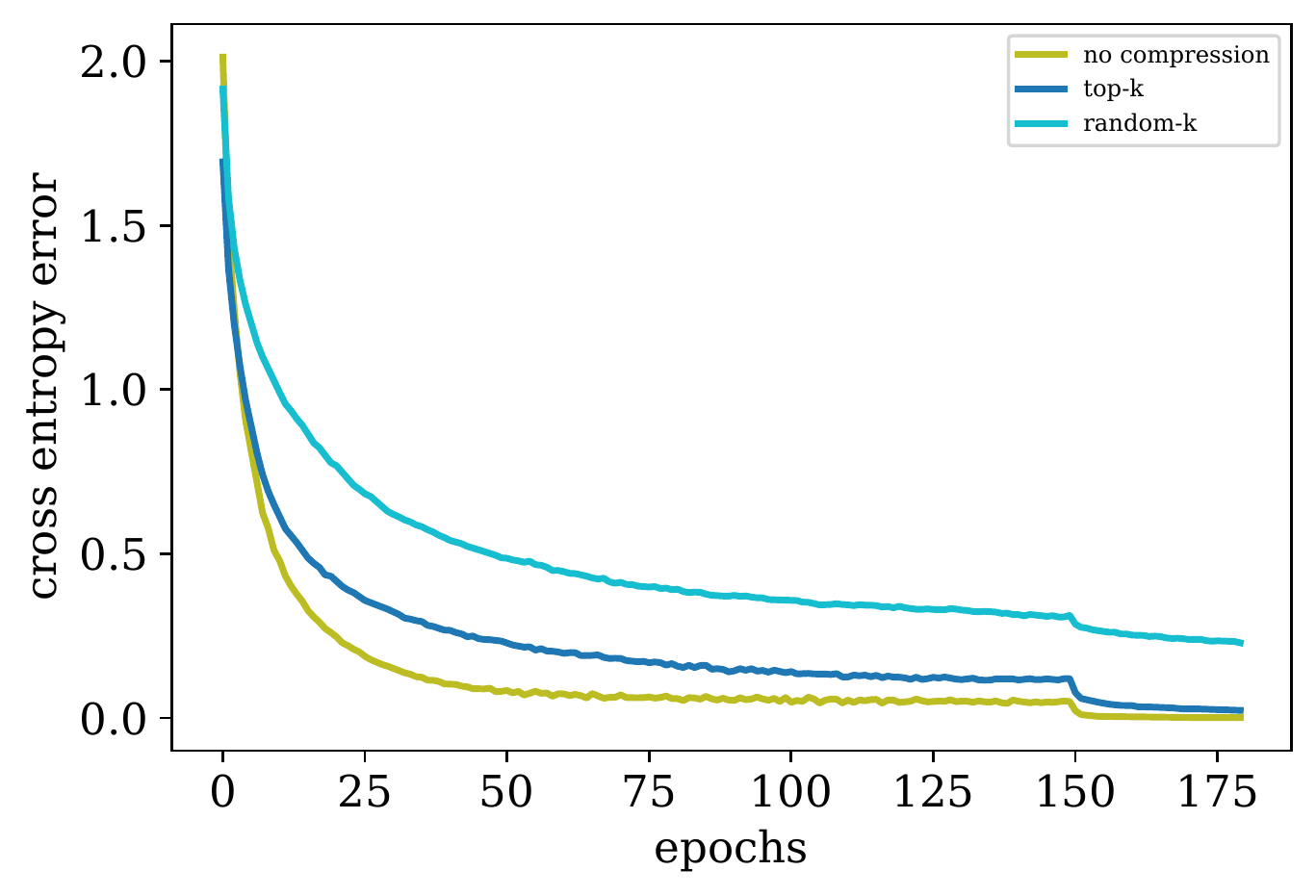}
    \caption{Training error when training ResNet18 on CIFAR-10 dataset. For the compression methods, only $1 \%$ of the coordinates were kept, i.e, $k = 0.01 \times d$.}
    \label{fig: DL_experiment_train}
\end{figure}

\subsection{DL Experiment on CIFAR-10 on Resnet18}
\label{subsec: DL}
We use the PyTorch framework \cite{pytorch_NEURIPS2019_9015} to empirically compare top-$k$ and random-$k$ with SGD without compression for training Resnet+BN+Dropout (Resnet18) \cite{Resnet_conf/cvpr/HeZRS16} on the CIFAR-10 dataset \cite{cifar-10article}.
Following the standard practice in compression schemes \cite{Deep_grad_comp_conf/iclr/LinHM0D18, ATOMO_conf/nips/WangSLCPW18}, we apply the compression layer-wise. %
We use SGD with momentum and we tune the momentum parameter $\beta$ and the learning rate separately for each setting (see Appendix~\ref{sec:parameters}). Although SGD with momentum is not covered by our analysis, we use this method since it is the state-of-the-art training scheme in deep learning and required to reach a good accuracy. In each setting we trained the neural network for 180 epochs and report the cross entropy error for the training data. \looseness=-1 %

\textbf{Discussion of results.} Figure \ref{fig: DL_experiment_train} shows the convergence. %
We can observe that similar to Figures  \ref{fig: constant_stepsize_randk} and \ref{fig: constant_stepsize_topk}, top-$k$ can still converge to the same error floor as SGD but with a slower rate.
Rand-$k$ is more severely impacted by the $\cO(\frac{d}{k})$ slowdown predicted by theory and would require more iterations to converge to higher accuracy. \looseness=-1

\section{Conclusion}

We provide a convergence analysis of SGD with biased gradient oracles. We reconcile insights that have been accumulated previously, explicitly or implicitly, by presenting a clean framework for the analysis of a general class of biased SGD methods.
We show concisely the influence of the bias and highlight important cases where SGD \emph{converges with biased oracles}. The simplicity of the framework will allow these findings to be broadly integrated, and the limitations we have identified may inspire targeted research efforts that can increase the robustness of SGD.

\clearpage
{\small 
\bibliographystyle{myplainnat}
\bibliography{relevant_papers}
}

\clearpage
\onecolumn
\appendix

\section{Deferred Proofs for Section~\ref{sec: biased gradient framework}}
\subsection{Proof of Lemma~\ref{lemma:decent-smoothness}}

By the quadratic upper bound~\eqref{eq: smoothness_equivalent_def} and Assumption~\ref{assumption:noise}:
\begin{align*}
    \E{f(\xx_{t+1})} & \leq
    f(\xx_t) - \gamma \lin{\nabla f(\xx_t), \E{\gg_t}} + \frac{\gamma^2 L}{2} \left(\E{\norm{\gg_t - \E{\gg_t}}^2} + \E{\norm{\E{\gg_t}}^2} \right) \\
    &= f(\xx_t) - \gamma \lin{\nabla f(\xx_t), \nabla f(\xx_t) + \bb_t} + \frac{\gamma^2 L}{2} \left(\E{\norm{\nn_t}^2} + \E{\norm{\nabla f(\xx_t) + \bb_t}^2} \right) \\
   &\leq f(\xx_t) - \gamma \lin{\nabla f(\xx_t), \nabla f(\xx_t) + \bb_t} + \frac{\gamma^2 L}{2} \left((M+1)\E{\norm{\nabla f(\xx_t) + \bb_t}^2} + \sigma^2 \right)
\end{align*}
By the choice of the stepsize, $\gamma \leq \frac{1}{(M+1)L}$, and Assumption~\ref{assumption: bias}:
\begin{align*}
    \E{f(\xx_{t+1})} &\leq f(\xx_t) + \frac{\gamma}{2} \left( -2 \lin{\nabla f(\xx_t),  \nabla f(\xx_t) + \bb_t} + \norm {\nabla f(\xx_t) + \bb_t}^2 \right) + \frac{\gamma^2L}{2} \sigma^2  \\
    &= f(\xx_t) + \frac{\gamma}{2} \left( - \norm{\nabla f(\xx_t)}^2 +  \norm{\bb_t}^2 \right)  + \frac{\gamma^2L}{2} \sigma^2  \\
    &\leq f(\xx_t) + \frac{\gamma}{2} \left(m-1 \right)\norm{\nabla f(\xx_t)}^2  + \frac{\gamma}{2}\zeta^2 + \frac{\gamma^2L}{2} \sigma^2\,.
\end{align*}
This concludes the proof. \hfill $\Box$

\subsection{Proof of Theorem~\ref{thm:smooth}}
By Lemma~\ref{lemma:new} we have that for any $T \geq 1$ and $\gamma \leq \frac{1}{(M+1)L}$ that it holds
\begin{align*}
 \Psi_{T} \leq \left(\frac{2F}{T\gamma (1-m)} + \frac{\gamma L \sigma^2}{1-m} \right) + \frac{\zeta^2}{1-m} =: \Theta(T,\gamma) \,,
\end{align*}
where we denote the right hand side by $\Theta(T,\gamma)$.

By choosing $T = \max \left\{ \frac{4(M+1)FL}{\epsilon(1-m)}, \frac{8 FL\sigma^2}{\epsilon^2(1-m)^2} \right\}$ and $\gamma=\min \left\{ \frac{1}{(M+1)L}, \frac{\epsilon (1-m)}{2 L \sigma^2} \right\}$, we observe $\Theta(T,\gamma) \leq \epsilon + \frac{\zeta^2}{1-m}$.
To see this, assume first that the minimum is attained for $\gamma = \frac{\epsilon(1-m)}{2 L \sigma^2}$. Then we verify
\begin{align*}
 \Theta(T,\gamma) = \frac{4FL\sigma^2}{T \epsilon (1-m)^2}  + \frac{\epsilon}{2} + \frac{\zeta^2}{1-m}  \leq \epsilon + \frac{\zeta^2}{1-m}
\end{align*}
by the choice of $T$. Similarly, if the minimum is attained for $\gamma = \frac{1}{(M+1)L}$, then $\frac{\sigma^2}{(M+1)(1-m)} \leq \frac{\epsilon}{2}$ and
\begin{align*}
 \Theta(T,\gamma) = \frac{2(M+1)FL}{T(1-m)} + \frac{\sigma^2}{(M+1)(1-m)} + \frac{\zeta^2}{1-m} \leq \frac{\epsilon}{2} + \frac{\epsilon}{2} + \frac{\zeta^2}{1-m}\,.
\end{align*}
Finally, it remains to note that we do not need to choose $\epsilon$ smaller than $\frac{\zeta^2}{1-m}$,  that is, we cannot get closer than $\frac{\zeta^2}{1-m}$ to the optimal solution. \hfill $\Box$

\subsection{Proof of Theorem~\ref{thm: convergence under PL condition}}
Starting from Equation~\eqref{eq:recusion}, we proceed by unrolling the recursion and using that $\sum_{i=0}^T (1-a)^i \leq \sum_{i=0}^\infty (1-a)^i = \frac{1}{a}$ for $0 < a < 1$:
\begin{align*}
    F_{T} \leq  \left(1-\gamma \mu(1-m) \right)^T F_0 + \frac{ \zeta^2}{2\mu (1-m)} + \frac{\gamma L \sigma^2}{2\mu (1-m)}\,. %
\end{align*}
Choose the stepsize $\gamma = \min\bigl\{\frac{1}{(M+1)L}, \frac{\epsilon \mu (1-m)}{L \sigma^2} \bigr\}$ and $T = \max\left\{ \frac{(M+1)L}{\mu (1-m)} \log\frac{2F_0}{\epsilon}, \frac{L\sigma^2}{\epsilon \mu^2 (1-m)^2} \log \frac{2F_0}{\epsilon} \right\}$. Then we verify:

If the minimum is attained for $\gamma= \frac{1}{(M+1)L}$, then $\frac{\sigma^2}{\mu (1-m)(M+1)} \leq \epsilon$, and it holds
\begin{align*}
 F_T \leq \left(1- \frac{\mu(1-m)}{(M+1)L} \right)^T F_0 + \frac{\sigma^2}{2\mu(1-m)(M+1)} + \frac{\zeta^2}{2 \mu (1-m)} \leq \frac{\epsilon}{2} + \frac{\epsilon}{2} + \frac{\zeta^2}{2 \mu (1-m)}
\end{align*}
by the choice of $T$.
Otherwise, if $\gamma= \frac{\epsilon \mu(1-m)}{L \sigma^2}$ then
\begin{align*}
 F_T \leq \left(1- \frac{\epsilon \mu^2(1-m)^2}{L\sigma^2} \right)^T F_0  + \frac{\epsilon}{2} + \frac{\zeta^2}{2 \mu (1-m)} \leq \epsilon +  \frac{\zeta^2}{2 \mu (1-m)}
\end{align*}
The claim follows again by observing that it suffices to consider $\epsilon \geq \frac{\zeta^2}{2\mu(1-m)}$ only. \hfill $\Box$.

\section{Deferred Proofs for Section~\ref{sec: special cases}}

\subsection{Top-$k$ sparsification}
\label{sec:topk}

\begin{definition}
For a parameter $1\leq k \leq d$, the operator $\topk : \R^d \to \R^d$ is defined for $\xx \in \R^d$ as:
\begin{align*}
    \left(\topk(\mathbf{x})\right)_{i}:=\left\{\begin{array}{ll}
(\mathbf{x})_{\pi(i)}, & \text { if } i \leq k \\
0 & \text { otherwise }
\end{array}\right.
\end{align*}
where $\pi$ is a permutation of $[d]$ such that $(\abs{\xx})_{\pi(i)} \leq (\abs{\xx})_{\pi(i+1)}$ for $i=1, \ldots, d-1$. In other words, the $\topk$ operator keeps only the $k$ largest elements of a vector and truncates the other ones to zero.
\end{definition}

We observe that for the $\topk$ operator, we have
\begin{align}
   0 \leq  \norm{\topk(\xx) - \xx}^2 \leq \frac{d-k}{d} \norm{\xx}^2\, \qquad \forall \xx \in \R^d\,, \label{eq:topk}
\end{align}
and both inequalities can be tight in general.

\begin{lemma}
Let $f \colon \R^d \to \R$ we differentiable, then 
$\topk(\nabla f(\xx))$ is a $(m,\zeta^2)$-biased gradient oracle for
\begin{align*}
m &= \left(1-\frac{k}{d}\right)\,, &
 \zeta^2 &= 0 \,.
\end{align*}
\end{lemma}
\begin{proof}
This follows directly from~\eqref{eq:topk}.
\end{proof}

\subsection{Random-$k$ sparsification}
\begin{definition}
For a parameter $1\leq k \leq d$, the operator $\randk : \R^d \times \Omega_k \to \R^d$, where $\Omega_k$ denotes the set of all $k$ element subsets of $[d]$, is defined for $\xx \in \R^d$ as:
\begin{align*}
    \left(\randk(\mathbf{x}, \omega)\right)_{i}:=\left\{\begin{array}{ll}
(\mathbf{x})_{i}, & \text { if } i \in \omega \\
0 & \text { otherwise }
\end{array}\right.
\end{align*}
where $\omega$ is a set of $k$ elements chosen uniformly at random from $\Omega_l$, i.e, $\omega \sim_{\text{u.a.r}} \Omega_k$.
\end{definition}

We observe that for the $\randk$ operator, we have
\begin{align}
  \E \norm{\randk(\xx) - \xx}^2 &= \frac{d-k}{d} \norm{\xx}^2\,, &
  \E \randk(\xx) &= \frac{k}{d} \xx \,, &
  &\forall \xx \in \R^d\,, \label{eq:randk}
\end{align}
and further
\begin{align}
 \E \randk(\aa+\bb) =\E \randk(\aa) + \E \randk(\bb) \label{eq:randklinear}
\end{align}
for any vectors $\aa,\bb \in \R^d$. Moreover
\begin{align}
 \norm{\E \randk(\xx)}^2 &= \frac{k^2}{d^2} \norm{\xx}^2 &
 \E \norm{\randk(\xx)}^2 &= \frac{k}{d} \norm{\xx}^2 \,.
\end{align}

\begin{lemma}
Let $\gg(\xx)$ be an unbiased gradient oracle with $(M_b,\sigma_b^2)$-bounded noise. Then $\randk(\gg(\xx))$ is a $(m,\zeta^2)$-biased gradient oracle with $(M,\sigma^2)$-bounded noise, for
\begin{align*}
 m &= \left(1-\frac{k}{d}\right)\,, &
 \zeta^2 &= 0\,,   &
 M &= \left((1+M_b) \frac{d}{k}-1 \right)\,, &
 \sigma^2 &= \frac{k}{d} \sigma_b^2\,.
\end{align*}
\end{lemma}
\begin{proof}
The gradient oracle can be written as $\gg(\xx)=\nabla f(\xx)+ \nn(\xx)$, for $(M_b,\sigma_b^2)$-bounded noise. We first estimate the bias
\begin{align}
 \norm{\E {\randk(\gg(\xx))} - \nabla f(\xx)}^2 &=
 \norm{\E {\randk(\nabla f(\xx)+\nn(\xx))} - \nabla f(\xx)}^2  \notag \\
 &\stackrel{\eqref{eq:randklinear}}{=} \norm{ \E \randk (\nabla f(\xx)) +  \E \randk(\nn(\xx)) - \nabla f(\xx) }^2  \notag \\
 & =  \norm{\E \randk(\nabla f(\xx)) -\nabla f(\xx)}^2 \notag \\
 &\stackrel{\eqref{eq:randk}}{=}  \left(1-\frac{k}{d}\right) \norm{\nabla f(\xx)}^2 \notag
\end{align}
where we used independence to conclude $\E [\randk(\nn(\xx))] = \0$.

For the noise, we observe $\E \norm{\randk(\nabla f(\xx))}^2 = \frac{d}{k}\norm{\E \randk(\nabla f(\xx))}^2 = \frac{d}{k} \norm{\E \randk(\gg(\xx))}^2$, and hence
\begin{align*}
  \E{\norm{ {\randk(\gg(\xx))} - \E{\randk(\gg(\xx))} }^2} 
    &= \E \norm{\randk(\nabla f(\xx) + \nn(\xx))}^2 - \norm{\E \randk(\gg(\xx))}^2 \\
    &= \E \norm{\randk(\nabla f(\xx))}^2  + \E \norm{\randk(\nn(\xx))}^2 - \norm{\E \randk(\gg(\xx))}^2 \\
    &= \left(\frac{d}{k}-1\right) \norm{\E \randk(\gg(\xx))}^2 + \frac{k}{d}\E \norm{\nn(\xx)}^2 \\
    &\leq \left(\frac{d}{k}-1\right) \norm{\E \randk(\gg(\xx))}^2 + \frac{k}{d} M_b \norm{\nabla f(\xx)}^2 +  \frac{k}{d} \sigma_b^2 \\
    &\leq \left(\frac{d}{k}-1 + \frac{d}{k} M_b \right)  \norm{\E \randk(\gg(\xx))}^2 +  \frac{k}{d} \sigma_b^2
\end{align*}
the statement of the lemma follows.
\end{proof}

\begin{remark}
These bounds are tight in general. 
When $k=d$, then we just recover $M=M_b$, $\sigma^2=\sigma_b^2$.
For the special case when $M_b=\sigma_b^2=0$, then we recover $M=\frac{d}{k}-1$ and $\sigma^2=0$, i.e, the randomness of the oracle would be only due to the random-$k$ sparsification.
\end{remark}

\paragraph{Biased versus unbiased sparsification.}
In the literature, often an unbiased version of random sparsification is considered, namely the scaled operator $\randk':= \frac{d}{k}\randk$. 

Up to the scaling by $\frac{d}{k}$, the $\randk'$ operator is therefore identical to the $\randk$ operator, and we would expect the same convergence guarantees, in particular with tuning for the optimal stepsize $\gamma'$ (as in our theorems). However, due to the constraint $\gamma \leq \frac{1}{L}$ as a consequence of applying the smoothness in Lemma~\ref{lemma:decent-smoothness}, this rescaling might not be possible if $\gamma' \in [ \frac{k}{dL}, \frac{1}{L}]$ (only if it is smaller). This is the reason, why the lemma below gives a slightly better estimate for the convergence with $\randk'$ than with $\randk$ (the optimization term, the term not depending on $\sigma^2$, is improved by a factor $\cO(\frac{d}{k})$).

However, we argue that this is mostly a technical issue, and not a limitation of $\randk$ vs. $\randk'$. Our Lemma~\ref{lemma:decent-smoothness} does not assume any particular structure on the bias and holds for arbitrary biased oracles. For the particular choice of $\randk$, we know that scaling up the update by a factor $\frac{d}{k}$ is possible in any case, and hence the condition on the stepsize could be relaxed to $\gamma \leq \frac{d}{kL}$ in this special case---recovering exactly the same convergence rate as with $\randk'$.

\begin{lemma}
Let $\gg(\xx)$ be an unbiased gradient oracle with $(M_b,\sigma_b^2)$-bounded noise. Then $\randk'(\gg(\xx))$ is a $(m,\zeta^2)$-biased gradient oracle with $(M,\sigma^2)$-bounded noise, for
\begin{align*}
 m &= 0\,, &
 \zeta^2 &= 0\,,   &
 M &= \left((1+M_b) \frac{d}{k}-1 \right)\,, &
 \sigma^2 &= \frac{d}{k} \sigma_b^2\,.
\end{align*}
\end{lemma}

\begin{proof}
The unbiasedness is easily verified. For the variance, observe
\begin{align*}
 \E \norm{\randk'(\gg(\xx)) - \nabla f(\xx) }^2  &=
 \E \norm{\randk'(\nabla f(\xx) + \nn(\xx))}^2 - \norm{\nabla f(\xx)}^2 \\
 &= \E \norm{\randk'(\nabla f(\xx))}^2 + \E \norm{\randk'(\nn(\xx))}^2 - \norm{\nabla f(\xx)}^2 \\
 &= \frac{d}{k} \norm{\nabla f(\xx)}^2 + \frac{d}{k} \E \norm{\nn(\xx)}^2  - \norm{\nabla f(\xx)}^2 \\
 &=  \left(\frac{d}{k}-1 +\frac{d}{k}M_b \right) \norm{\nabla f(\xx)}^2 +  \frac{d}{k}\sigma_b^2 \qedhere
\end{align*}
\end{proof}

\subsection{Zeroth-Order Gradient}

\begin{lemma}
let $\gg_{GS}$ be a zeroth-order gradient oracle as defined in \eqref{eq: grad free oracle}. Then this oracle can be described as a Biased Gradient Oracle defined in Definition \ref{def: biased gradient oracle} with parameters:
\begin{align*}
    m = 0, \quad \zeta^2 \leq \frac{\tau^2}{4}L^2(d+3)^3, \quad M \leq 4(d+4), \quad \sigma^2 \leq 3\tau^{2}L^{2}(d+4)^{3}
\end{align*}
\end{lemma}
\begin{proof}
From \citet[Lemma 3]{Nesterov:2017:RGM:3075317.3075511} we can bound the bias for $\gg_{GS}$ as follows:
\begin{align*}
    \norm{\bb(\xx)}^2
    &= \norm{\E {\gg_{GS}(\xx)} - \nabla f(\xx)}^2 \\
    &\leq \frac{\tau^2}{4}L^2(d+3)^3
\end{align*}
Hence we can conclude that Assumption \ref{assumption: bias} holds with the choice $m=0$, and $\zeta^2 \leq \frac{\tau^2}{4}L^2(d+3)^3$. Moreover, 
\citet[Theorem 4]{Nesterov:2017:RGM:3075317.3075511} provide an upper bound on the second moment of the zeroth-order gradient oracle $\gg_{GS}$:
\begin{align*}
    \mathbb{E} \norm{\gg_{GS} (\xx)}^2 & \leq 3\tau^{2}L^{2}(d+4)^{3} +  4(d+4) \norm{\E {\gg_{GS}(\xx)}}^{2} \\
    & = 3\tau^{2}L^{2}(d+4)^{3} +  4(d+4) \norm{\nabla f(\xx) + \bb(\xx)}^{2}\,.
\end{align*}
The second moment is an upper bound on the variance. Hence we can conclude that Assumption~\ref{assumption:noise} hold with the choice
\begin{align*}
    \sigma^2 &= 3 \tau^{2} L^{2}(d+4)^{3}, \qquad M = 4(d+4)\,. \qedhere
\end{align*}
\end{proof}
\paragraph{Convergence rate for $\gg_{GS}$ under PL condition.} 
Using the result from Theorem \ref{thm: convergence under PL condition}, we can conclude that under PL condition,
\begin{align*}
    \cO \left(\frac{dL}{\mu} \log \frac{1}{\epsilon} + \frac{\tau^2 L^3 (d+1)^3}{\mu^2 \epsilon + \mu \tau^2 L^2(d+1)^3}\right)
\end{align*}
iterations is sufficient for zeroth-order gradient descent to reach $F_T \leq \epsilon + \frac{\zeta^2}{\mu}$.
Observe that for any reasonable choice of $\epsilon =\Omega \bigl( \frac{\zeta^2}{\mu}\bigr)$, the second term is of order $\cO\bigl(\kappa \bigr)$ only, and hence the rate is dominated by the first term. We conclude, that zeroth-order gradient descent requires $\cO \left(d \kappa \log \frac{1}{\epsilon}\right)$ iterations to converge to the accuracy level of the bias term $\frac{\zeta^2}{\mu} = \cO \left(\frac{\tau^2 L^2 d^3}{\mu}\right)$ which cannot be surpassed.

The convergence rate matches with the rate established in \citep[Theorem 8]{Nesterov:2017:RGM:3075317.3075511}, but  in our case the convergence radius (bias term) is $\cO(d)$ times worse than theirs. This shows that in general the rates given in Theorem~\ref{thm: convergence under PL condition} can be improved if stronger assumptions are imposed on the bias term or if it is known how the gradient oracle (and bias) are generated, for instance by Gaussian smoothing.

\subsection{Inexact gradient oracle}
First, we give the definition of a $(\delta,L)$ oracle as introduced in~\cite{devolder_inexact_oracle_DBLP:journals/mp/DevolderGN14}.
\begin{definition}[($\delta, L$)-oracle \cite{devolder_inexact_oracle_DBLP:journals/mp/DevolderGN14}]
Let function $f$ be convex on $Q$, where $Q$ is a closed convex set in $\mathbb{R}^d$. We say that $f$ is equipped with a first-order $(\delta, L)$-oracle if for any $\yy \in Q$ we can compute a pair $(\tilde{f}(\yy), \tilde{\gg}(\yy))$ such that:
\begin{align*}
    0 \leq f(\xx)-\left(\tilde{f}(\yy)+\lin{ \tilde{\gg}(\yy), \xx-\yy}\right) \leq \frac{L}{2}\norm{\xx-\yy}^{2}+\delta\,, \qquad \forall \xx \in Q\,.
\end{align*}
\end{definition}
The constant $\delta$ is called the accuracy of the oracle. A function $f$ is $L$-smooth if and only if it admits a $(0, L)$-oracle. However the class of functions admitting a $(\delta, L)$-oracle is strictly larger and includes non-smooth functions as well.
\paragraph{Stochastic inexact gradient oracle.} \citet{smooth_convex_opt_devolder_RePEc:cor:louvco:2011070} generalized the notion of $(\delta, L)$-gradient oracle to the stochastic case. Instead of using the pair $(\tilde{f}(\xx), \tilde{\gg}(\xx))$, they use the stochastic estimates $(\tilde{F}(\xx, \xi), \tilde{G}(\xx, \xi))$ such that:
\begin{align*}
    &\EEb{\xi}{\tilde{F}(\xx, \xi)} = \tilde{f}(\xx)\,, \\
    &\EEb{\xi}{\tilde{G}(\xx, \xi)} = \tilde{\gg}(\xx)\,, \\
    &\EEb{\xi}{\norm{\tilde{G}(\xx, \xi) - \tilde{\gg}(\xx)}^2} \leq \sigma^{2}\,,
\end{align*}
where $(\tilde{f}(\xx), \tilde{\gg}(\xx))$ is a $(\delta,L)$-oracle as defined above.
From the third inequality we can conclude that $M=0$. Moreover the bias term can be upper bounded by:
\begin{align*}
    \norm{\bb(\xx)}^2 = \norm{\EEb{\xi}{\tilde{G}(\xx, \xi)} - \nabla f(\xx)}^2 = \norm{\tilde{\gg}(\xx) - \nabla f(\xx)}^2 \leq 2 \delta L
\end{align*}
Therefore we can conclude that $m=0$, and $\zeta^2 = 2\delta L$.

\subsection{Biased compression operators}
Deriving the bounds for arbitrary compressors follows similarly as in Section~\ref{sec:topk} when observing that a $\delta$-compressor satisfies~\eqref{eq:topk} with factor $(1-\delta)$ instead of $\left(1-\frac{k}{d}\right)$ as for the top-$k$ compressor (the top-$k$ compressor is a $\delta$-compressor for $\delta = \frac{k}{d}$).

\section{Details on the deep learning experiment}
\label{sec:parameters}
The hyper-parameter choices used in our deep learning experiments (Section \ref{subsec: DL}) are summarized in Table \ref{tab: hyper-params DL}. We tuned the hyper-parameters over the range $[0.1, 0.2, 0.3, 0.4, 0.5]$ for learning rate and $[0.1, 0.25, 0.5, 0.75, 0.9]$ for momentum and choose the combination of hyper-parameters that gave the least test error. In all of the settings we reduced the learning rate at the epoch 150 with a factor of 10 and we trained the network with a batch size of $256$ and a $L2$ penalty regularizer equal to $10^{-4}$.

\begin{table}[!ht]
\caption{Hyper-parameters used in the DL experiment}
\centering
\begin{tabular}{||ccc||ll}
\cline{1-3}
\textbf{method} & \textbf{lr}             & \textbf{momentum $\beta$} &  &  \\ \cline{1-3}
SGD             & \multicolumn{1}{l}{0.1} & 0.9                       &  &  \\
top-$k$         & 0.1                     & 0.75                       &  &  \\
rand-$k$        & 0.1                     & 0.75                       &  &  \\ \cline{1-3}
\end{tabular}
\label{tab: hyper-params DL}
\end{table}

\begin{figure}[!ht]
    \centering
    \includegraphics[width=0.6\linewidth]{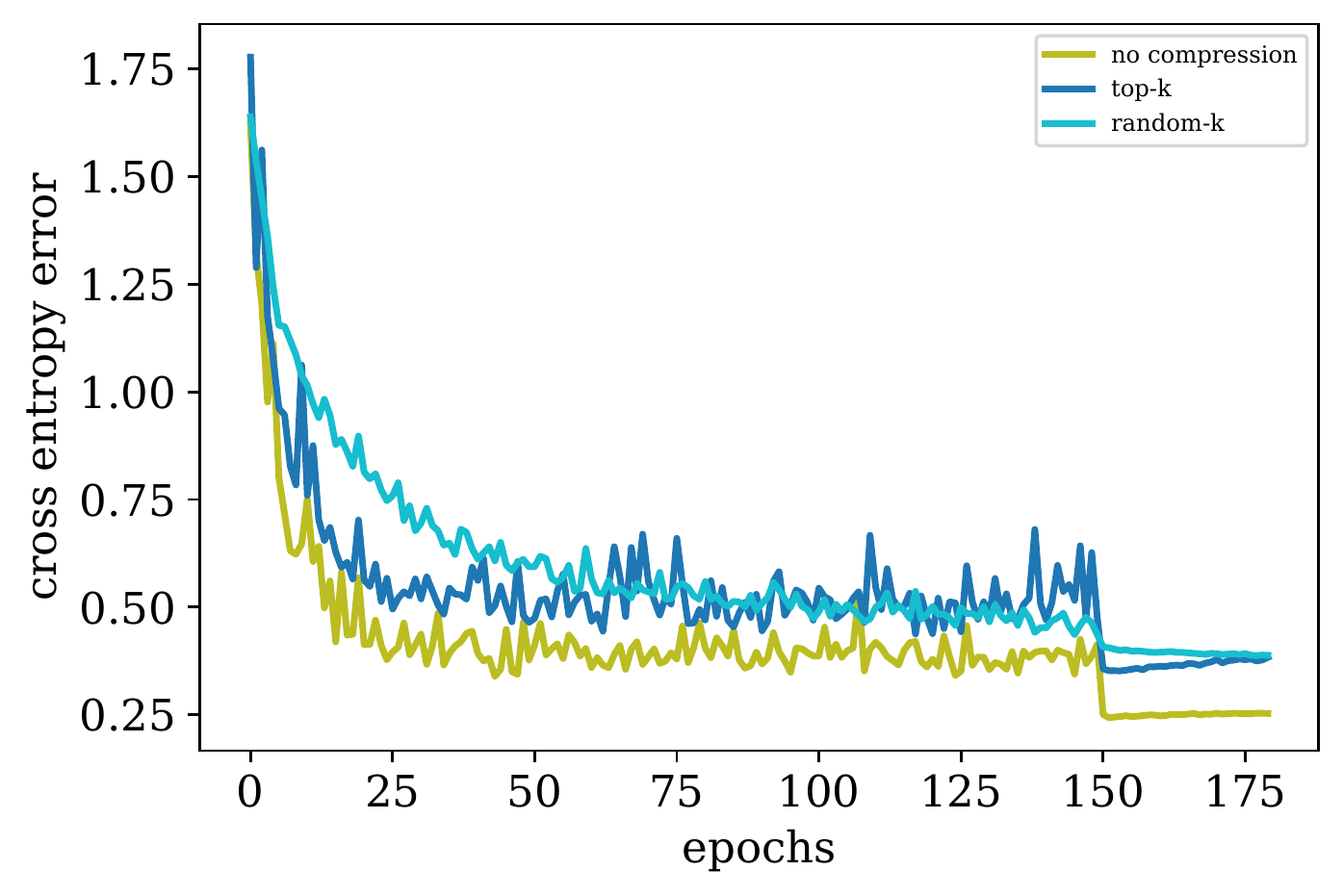}
    \caption{Test error when training ResNet18 on CIFAR-10 dataset. For the compression methods, only $1 \%$ of the coordinates were kept, i.e, $k = 0.01 \times d$.}
    \label{fig: DL_experiment_test}
\end{figure}

\end{document}